\documentclass[verbose=true,letterpaper,margin=1.5in]{article}

\usepackage{amsmath}
\usepackage{times}  
\usepackage{helvet}  
\usepackage{courier}  
\usepackage[hyphens]{url}  
\usepackage{graphicx} 
\usepackage{amsmath} 
\usepackage{amssymb}
\usepackage{amssymb}
\usepackage{pifont}
\newcommand{\cmark}{\ding{51}}%
\newcommand{\xmark}{\ding{55}}%
 \usepackage{amsfonts}
 \usepackage{amssymb}
 \usepackage{amsbsy}
 \usepackage{mathrsfs,latexsym,url,epsfig,graphicx}
 \usepackage{amsthm,array}
 \usepackage{multirow}
 \usepackage {hyperref}
\usepackage{subfigure}
\usepackage{booktabs} 
\usepackage{algorithmic}
\usepackage{ulem}
\usepackage[preprint,nonatbib]{neurips_2022}

\usepackage[utf8]{inputenc} 
\usepackage[T1]{fontenc}    
\usepackage{url}            
\usepackage{booktabs}       
\usepackage{amsfonts}       
\usepackage{nicefrac}       
\usepackage{microtype}      
\usepackage{xcolor}         
\usepackage[sort&compress,round,comma,numbers]{natbib}
\newtheorem{theorem}{Theorem}

\newtheorem{lemma}{Lemma}
\newtheorem{assumption}{Assumption}
\usepackage{algorithm}  
\usepackage{amsmath}  
\usepackage{natbib}
\setcitestyle{square, comma, numbers,sort&compress, super}
\floatname{algorithm}{Algorithm}  
\renewcommand{\algorithmicrequire}{\textbf{Input:}}  
\renewcommand{\algorithmicensure}{\textbf{Output:}}  

\title{A Novel Noise Injection-based Training Scheme for Better Model Robustness}

\author{
Zeliang Zhang\thanks{The first and second authors contributed equally to this work.}\\
University of Rochester\\
\and
Jinyang Jiang\textsuperscript{$\star$}\\
Peking University\\
\and
Minjie Chen\\
Huaru Tech
\and
Zhiyuan Wang\\
HUST\\
\and
Yijie Peng\\
Peking University\\
\and
Zhaogfei Yu\\
Peking University\\
}

\begin{document}

\maketitle

\begin{abstract}
 
Noise injection-based method has been shown to be able to improve the robustness of artificial neural networks in previous work. In this work, we propose a novel noise injection-based training scheme for better model robustness. Specifically, we first develop a likelihood ratio method to estimate the gradient with respect to both synaptic weights and noise levels for stochastic gradient descent training. Then, we design an approximation for the vanilla noise injection-based training method to reduce memory and improve computational efficiency. Next, we apply our proposed scheme to spiking neural networks and evaluate the performance of classification accuracy and robustness on MNIST and Fashion-MNIST datasets. Experiment results show that our proposed method achieves a much better performance on adversarial robustness and slightly better performance on original accuracy, compared with the conventional gradient-based training method.

\end{abstract}

\section{Introduction}

Artificial neural networks (ANNs) have increasingly more successful applications, such as face recognition~\citep{zhao2003face,li2020review,nawaz2020artificial, song2021talking}, voice verification~\citep{jung2019self,faisal2019specaugment,qian2021audio,song2021tacr} and automation vehicles~\citep{spielberg2019neural,sung2021training,safiullin2020method}.  While various studies design algorithms and architectures to improve the prediction accuracy of ANNs, there is much less work focusing on improving robustness~\citep{mangal2019robustness}. 

For a given input $x$,  ANNs should output a robustness prediction for all inputs $\hat{x}$ within a ball of radius $\delta$ centered at $x$. ANNs has exhibited poor robustness, leading to the unfairness of outcomes~\citep{bellamy2018ai,trewin2019considerations,madaio2020co}, leakage of private information~\citep{chase2017private,chang2018privacy,hitaj2017deep}, and susceptibility to input perturbations~\citep{goodfellow2015explain,kurakin2018adversarial,ilyas2019adversarial}. Our work focuses on improving the robustness of ANNs against input  perturbations, especially for adversarial perturbations. Recent studies find that ANNs are extremely vulnerable to adversarial perturbations, the crafted noise added on original examples, which is unperceivable to human eyes but can mislead the ANNs. Adversarial attacks pose a severe threat to the security of ANNs in various applications ~\citep{liu2016delving,song2018physical,xie2020real,zhang2021generating, capito2021modeled}. 

Noise injection-based methods can effectively improve the robustness of ANNs under adversarial attacks. \citet{you2019adversarial} add Gaussian random noises to layers in ANN to improve robustness, which can be viewed as a regularization method to alleviate over-fitting. \citet{Xiao2021NoiseOpt} further optimize noise levels to better defend against input perturbations without loss of classification accuracy. We further explore how to leverage added noises to optimize the parameters of models for better robustness.

The likelihood ratio (LR) method is an unbiased stochastic gradient estimation technique, which has been applied to many simulation optimization problems ~\cite{hong2009estimating,peng2015discontinuity,peng2022new}.  \citet{peng2022new} use the technique to train ANNs and achieve an accuracy performance comparable to the backpropagation (BP) method. Yet, many potentials remain to be discovered, such as improving efficiency and scalability.

We propose a novel noise injection-based training scheme more efficient than the original LR to optimize model parameters
and apply it to train spiking neural networks (SNNs). We evaluate the accuracy and robustness of our proposed method on MNIST and Fashion-MNIST under various types of adversarial attacks. The experiment results show that our method leads to much better robustness against adversarial attacks and slightly better accuracy on original samples.

\section{Related Work}

\textbf{Neural Networks and Training Methods:} 
Many optimization methods, such as backpropagation (BP)~\citep{rumelhart1986learning}, heuristic algorithm~\citep{manoharan2020population} and some bio-inspired approaches~\citep{lee2015difference, ororbia2019spiking,ororbia2019biologically}, have been proposed to train artificial neural networks (ANN).  Among them,  BP is the most popular method  to train conventional ANNs.  And its variant, spatio-temporal BP (STBP) algorithm~\citep{wu2019direct,wu2018spatio}, can also be applied to train spiking neural networks (SNN). For conventional ANNs, BP algorithm computes the gradient of the loss function with respect to weight parameters by the chain rule, which requires the differentiability of loss function and activation function. SNN uses a discontinuous threshold function as the activation function and adopts certain memory mechanism to the feed forward process. STBP algorithm uses approximated derivative of the threshold function and combines the layer-by-layer spatial domain and  the timing-dependent temporal domain for stochastic gradient descent. However, both BP and STBP lead to poor robustness. Our proposed method improves model robustness for conventional ANNs and SNNs.

\textbf{Adversarial Attack:} After \citet{szegedy2016inceptionv3} find the vulnerability of ANNs against adversarial attacks, various methods generating adversarial samples have been proposed, including gradient-based attack~\citep{goodfellow2014explaining,kurakin2018adversarial, madry2018towards}, transfer-based attack~\citep{dong2018boosting, xie2019improving,wei2019transferable}, score-based attack~\citep{Ru2020BayesOpt,meunier2019yet,du2019query}, decision-based attack~\citep{brendel2017decision, li2020qeba,wang2021triangle}. Existing adversarial attacks can be generally  categorized into two types: 1) white-box attack, where attackers have access to the architecture and parameters of the target ANN model; 2) black-box attack, where attackers can only access the output of the target model. Gradient-based attack is one of the most investigated white-box attack method. \citet{goodfellow2014explaining} propose a fast gradient sign method (FGSM) to generate adversarial samples. \citet{kurakin2016adversarial} propose a multi-step FGSM for a basic iterative method (BIM). \citet{madry2018towards} propose a project gradient descent (PGD) attack which applies random initialization to BIM and achieves a remarkable improvement on the success rate of the attack. \citet{dong2018boosting} further use momentum to improve BIM and propose momentum iterative method (MIM) to enhance the performance of gradient-based attack. \citet{li2018nattack} propose an evolution gradient search based block-box adversarial attack against models with defense. Transfer-based attacks do not need to access the model, which makes it popular in real world applications. \citet{xie2019improving} apply diversity image transformations (DIM) on inputs to generate transferable adversarial samples,
and \citet{Dong_2019_CVPR} propose a translation-invariant method (TIM) which employs pre-defined kernels to convolve the gradient for improving transferable adversarial attacks against ANNs with defense mechanism. In our work, we will evaluate the adversarial robustness of models  under these attacks.

\section{Methodology}
We first introduce a general framework shown in Fig.~\ref{framework} for training ANNs and SNNs using the LR method. Then, we propose an approximation method to reduce memory cost and improve computation efficiency, and provide the proof of convergence. In the end, we extend the method to train SNN.

\begin{figure}[h]
\centering
	\includegraphics[scale=0.4]{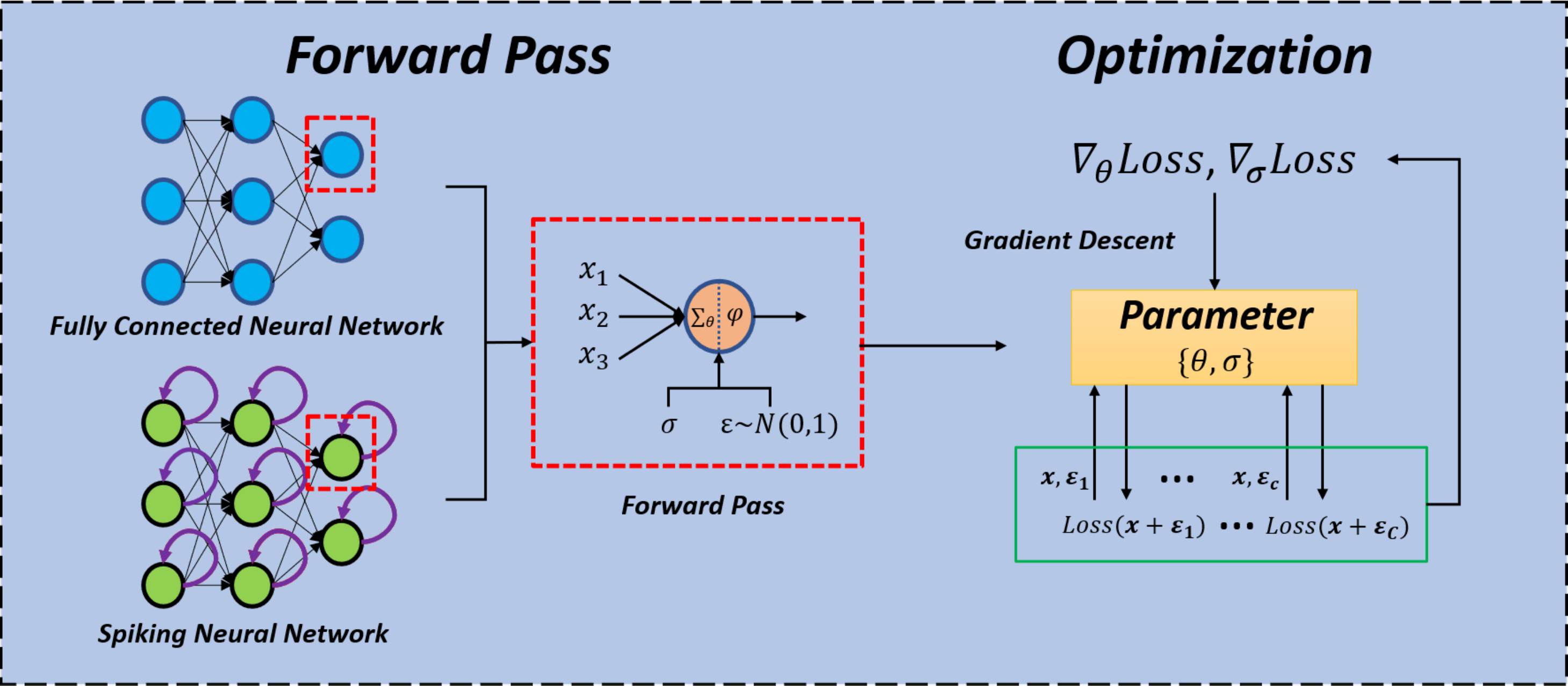}
	\caption{The framework of noise injection-based training method.}
	\label{framework}
\end{figure}

\subsection{The Likelihood Method for Training ANNs}
We denote $L$ as the number of layers in ANNs and $m_{l}$ as the number of neurons in the $l$-th neural layer, $l \in [1, 2, ..., L]$. For the input ${X}^{(0)} \in \mathbb{R}^{m_{0}}$, we have the output of $l$-th layer ${X}^{(l)}=[{x}^{(l)}_{1}, {x}^{(l)}_{2}, ..., {x}^{(l)}_{m_{l}}] \in \mathbb{R}^{m_{l}}$.

Suppose we have $N$ inputs for the network, denoted as ${X}^{(0)}(n)$, $n=1, 2, ..., N$. For the $n$-th input, the $i$-th output at the $l$-th layer can be given by
\begin{align*}
    {x}^{(l+1)}_{i}(n) &= \varphi({v}^{(l)}_{i}(n)),\\ {v}^{(l)}_{i}(n)&=\sum_{j=0}^{m_{l}}\theta^{(l)}_{i,j}{x}^{(l)}_{j}(n)+{\sigma}^{(l)}_{i}{\varepsilon}^{(l)}_{i}(n),
\end{align*}   
where ${x}^{(l)}_{j}(n)$ is the $j$-th input at the  $l$-th layer for the $n$-th data, $\theta^{(l)}_{i,j}$ is the synaptic weight in the $i$-th neuron for the j-th input at the $l$-th layer,  ${v}^{(l)}_{i}$ is the $i$-th logit output at the $l$-th layer, $\varphi$ is the activation function, ${\varepsilon}^{(l)}_{i}(n)$ is an independent random noise following standard normal distribution added to the $i$-th neuron at the $l$-th layer, and $\sigma_i^{(l)}$ is the standard deviation to scale up or down the noise. We let ${x}_0^{(l)}(n)\equiv 1$ and then $\theta^{(l)}_{i,0}$ is the bias term in the linear operation of the $i$-th neuron at the $l$-th layer.  

For the $n$-th input ${X}^{(0)}(n)$ with label ${O}{(n)}=[{o}_{1}(n), {o}_{2}(n), ..., {o}_{m_{L}}(n)] \in \mathbb{R}^{m_{L}}$, we have a loss $\mathcal{L}_n(\theta,\sigma)$. In classification tasks,  the loss function is usually the cross entropy computed by
\begin{align*}
    \mathcal{L}_n(\theta,\sigma) =\mathcal{L}(\theta,\sigma; {X}^{(L)}(n), {O}(n)) = -\sum_{i=1}^{m_{L}}{o}_{i}(n)\log\left( p_{i}({X}^{(L)}(n))\right),
\end{align*}  
where
\begin{align*}
    p_{i}({X}^{(L)}(n))=\frac{\exp{({x}^{(L)}_{i}}(n))}{\sum_{j=1}^{m_{L}}\exp{({x}^{(L)}_{j}(n)})}.
\end{align*}
Training ANNs is to solve the following optimization problem:
\begin{align*}
    \min\limits_{(\theta,\sigma)\in\Theta\times\Sigma}\mathcal{L}(\theta,\sigma):=  \frac{1}{N}\sum_{n=1}^N\mathbb{E}\left[\mathcal{L}_n(\theta,\sigma)\right],
\end{align*}
and a basic approach to solve it is by stochastic gradient descent (SGD) algorithms. Let $\omega=(\theta,\sigma)$ and $\Omega = \Theta\times\Sigma$. The SGD algorithm updates $\omega$ by
\begin{align*}
    \omega_{k+1} = \Pi_{\Omega}(\omega_k-\lambda_kg_k),\quad g_k = \frac{1}{b}\sum_{n\in B_k} g_{k,n},
\end{align*}
where $g_{k,n}$ is an unbiased estimator for the gradient of $\mathbb{E}\left[\mathcal{L}_n(\theta,\sigma)\right]$, $B_k=\{n_k^1,\cdots,n_k^b\}$ is a set of indices in a mini-batch randomly drawn from the $N$ data points, $\lambda_k$ is the learning rate, and $\Pi_{\Omega}$ is the projection onto $\Omega$ that bounds the values of $\omega$ in order to achieve convergence of SGD.

By an LR method, we have an unbiased gradient estimation of $\mathbb{E}\left[\mathcal{L}_n(\theta,\sigma)\right]$ with respect to ANN parameters as follows:
\begin{equation}\label{gradient}
\begin{aligned}
    \frac{\partial \mathbb{E}\left[\mathcal{L}_n(\theta,\sigma)\right]}{\partial \theta_{i,j}^{(l)}}&=\mathbb{E}\left[\mathcal{L}_n(\theta,\sigma){x}^{(l)}_{j}(n)\frac{\varepsilon^{(l)}_{i}(n)}{{\sigma}^{(l)}_{i}}\right],\\
    \frac{\partial \mathbb{E}\left[\mathcal{L}_n(\theta,\sigma)\right]}{\partial {\sigma}^{(l)}_{i}}&=\mathbb{E}\left[\mathcal{L}_n(\theta,\sigma)\frac{1}{{\sigma}^{(l)}_{i}}(\varepsilon^{(l)}_{i}(n)^2-1)\right].
\end{aligned}
\end{equation} 
The detailed derivation is presented in the supplementary.
To reduce variance of the LR estimator, we need to feed each data point $X^{(0)}(n)$ into the ANN multiple times and use the sample mean to estimate the expectation in Eq. (\ref{gradient}):
\begin{align*}
    \frac{1}{C}\sum_{c=1}^C \mathcal{L}_{n,c}(\theta,\sigma){x}^{(l)}_{j}(n,c)\frac{\varepsilon^{(l)}_{i}(n,c)}{{\sigma}^{(l)}_{i}},\text{ and }\frac{1}{C}\sum_{c=1}^C \mathcal{L}_{n,c}(\theta,\sigma)\frac{1}{{\sigma}^{(l)}_{i}}(\varepsilon^{(l)}_{i}(n,c)^2-1),
\end{align*}
where $c$ is the index of $C$ replications for the same input data point, and $x^{(0)}_j(n,c)=x^{(0)}_j(n)$ for all $c$. With more replications, the gradient estimation $g_{k,n}$ will be more accurate.

Tab.~\ref{tab:comp} compares our method with conventional BP algorithm on various aspects, including parallelism, variance, and generalization capability. We elaborate the strength and weakness of our method as follows.
\begin{table}[h]
\centering
\caption{Comparison of features between BP and LR.}
\begin{tabular}{|l|l|l|l|}
\hline
Feature           & BP   & LR  \\ \hline
parallelism     & Low   & High \\ 
Path-independent  & \xmark      & \cmark    \\ 
\begin{tabular}[c]{@{}l@{}}Generalization for \\ activations and loss functions\end{tabular} & \xmark                                  & \cmark                 \\ 
No Gradient Issue & \xmark       & \cmark    \\ 
Computation Complexity                                                                       & $\mathcal{O}(n^2)$  & $\mathcal{O}(n^2)$ \\ 
Variance & Low   & High \\ \hline
\end{tabular}\label{tab:comp}
\end{table}

\textbf{Parallelism:}  
Our method is more efficient for parallel computing than BP and STBP. Conventional BP method rely on the chain rule to propagate the errors backward using matrix multiplication from the loss value to the first neuron layers. The gradient of our method can be estimated without backpropagation by simply computing the element-wise product of the input, the added noise and the loss value. Thus, the training process using Eq. (\ref{gradient}) only involves a forward pass and the computation of the gradient could be paralleled for each neuron, which can fully leverage the computation power of GPUs. 

\textbf{No gradient issues:} 
In BP, the gradient of activation function is multiplied with the residual error for each layer, which might result in vanishing or explosive gradient issue. Our method does not have the issue because it computes the gradient of the parameters within each neuron.

\textbf{Generalization:} 
Our method can apply to any activation function and loss function, and can treat the structure of ANN as a block-box.

However, we need to store $\mathbb{X}^{(l)}(n)$ and $\varepsilon^{(l)}_{i}(n)$, for $l = 1, 2, ..., L$, in the computation of LR. Limited memory hinders the application of the method to complicated examples.  Thus, we further propose an approximation of the LR method for reducing the memory and computation costs.

\subsection{An Approximation to the Likelihood Ratio Method}
\label{sec:simplify_lr}
To reduce the memory cost in gradient estimation, we substitute the term inside the expectation in the first line of Eq. (\ref{gradient}). The original gradient estimation can be presented as 
\begin{align}\label{org_grad}
    g_k = \frac{1}{b}\sum_{n\in B_k}\mathcal{L}_n(\theta_k,\sigma_k)Z_n,
\end{align}
where $Z_n = (Z_n^{\theta},Z_n^{\sigma})$,  $(Z_n^{\theta})_{i,j}^{(l)}={x}^{(l)}_{j}(n)\frac{\varepsilon^{(l)}_{i}(n)}{{\sigma}^{(l)}_{i}}$ and $(Z_n^{\sigma})_{i}^{(l)}=\frac{1}{{\sigma}^{(l)}_{i}}(\varepsilon^{(l)}_{i}(n)^2-1)$. Then we propose the following surrogate ascent direction:
\begin{align}\label{sign_grad}
    \tilde{g}_k = \frac{1}{b}\sum_{n\in B_k}\mathcal{L}_n(\theta_k,\sigma_k)(\text{sign}(Z_n^{\theta}),Z_n^{\sigma}),
\end{align}
where we denote 
\begin{align*}
    \text{sign}({x}) = \begin{cases}
                1,  \qquad &x \geq 0, \\
                -1, \qquad &x \textless 0.
                \end{cases}
\end{align*}
SGD with the surrogate ascent direction can be written as
\begin{align}\label{new_sgd}
    \omega_{k+1} = \Pi_{\Omega}(\omega_k-\lambda_k\tilde{g}_k).
\end{align}

Denote the flattened $Z_n$ and $\omega_k$ as $[z_n^1,\cdots,z_n^D]^{\top}$ and $[\omega_k^1,\cdots,\omega_k^D]^{\top}$, where the first $D_0$ dimensional terms correspond to the coordinates of the $\theta_k$ and $1<D_0<D$, and the rest terms correspond to the coordinates of $\sigma_k$. Similar to LR, $\tilde{g}_k$ is an unbiased estimator of
\begin{equation}\label{approx_gradient}
    \begin{aligned}
    &\mathcal{J}(\omega_k):=\frac{1}{N}\sum_{n=1}^N\mathbb{E}\left[\mathcal{L}_n(\omega_k)(\text{sign}(Z_n^{\theta}),Z_n^{\sigma})\right]\\
    & = \frac{1}{N}\sum_{n=1}^N\left[
    \frac{\partial}{\partial\omega_{k,1}}\mathbb{E}\left[\frac{\mathcal{L}_n(\omega_k)}{\vert Z_{n,1}\vert}\right],\cdots,
    \frac{\partial}{\partial\omega_{k,D_0}}\mathbb{E}\left[\frac{\mathcal{L}_n(\omega_k)}{\vert Z_{n,D_0}\vert}\right],
    \frac{\partial\mathbb{E}\left[\mathcal{L}_n(\omega_k)\right]}{\partial\omega_{k,D_0+1}},\cdots,
    \frac{\partial\mathbb{E}\left[\mathcal{L}_n(\omega_k)\right]}{\partial\omega_{k,D}}\right]^{\top}.
\end{aligned}
\end{equation}

The detailed derivation is presented in the supplementary. 
Due to the existence of extra term $\vert Z_{n,d}\vert^{-1}$ in the first $D_0$ components, the objective function is different from the original one. However, Fig. \ref{fig:viz_opt} shows that the influence of the extra term may be negligible in the early stage of training. To reduce total computational burden, we can update parameters in the direction of $\tilde{g}_k$ in the early stage and then switch to $g_k$.

\begin{figure}[htbp]
\subfigure[Trace of optimization of our method.] 
{
	\begin{minipage}{7cm}
	\centering          
	\includegraphics[scale=0.48]{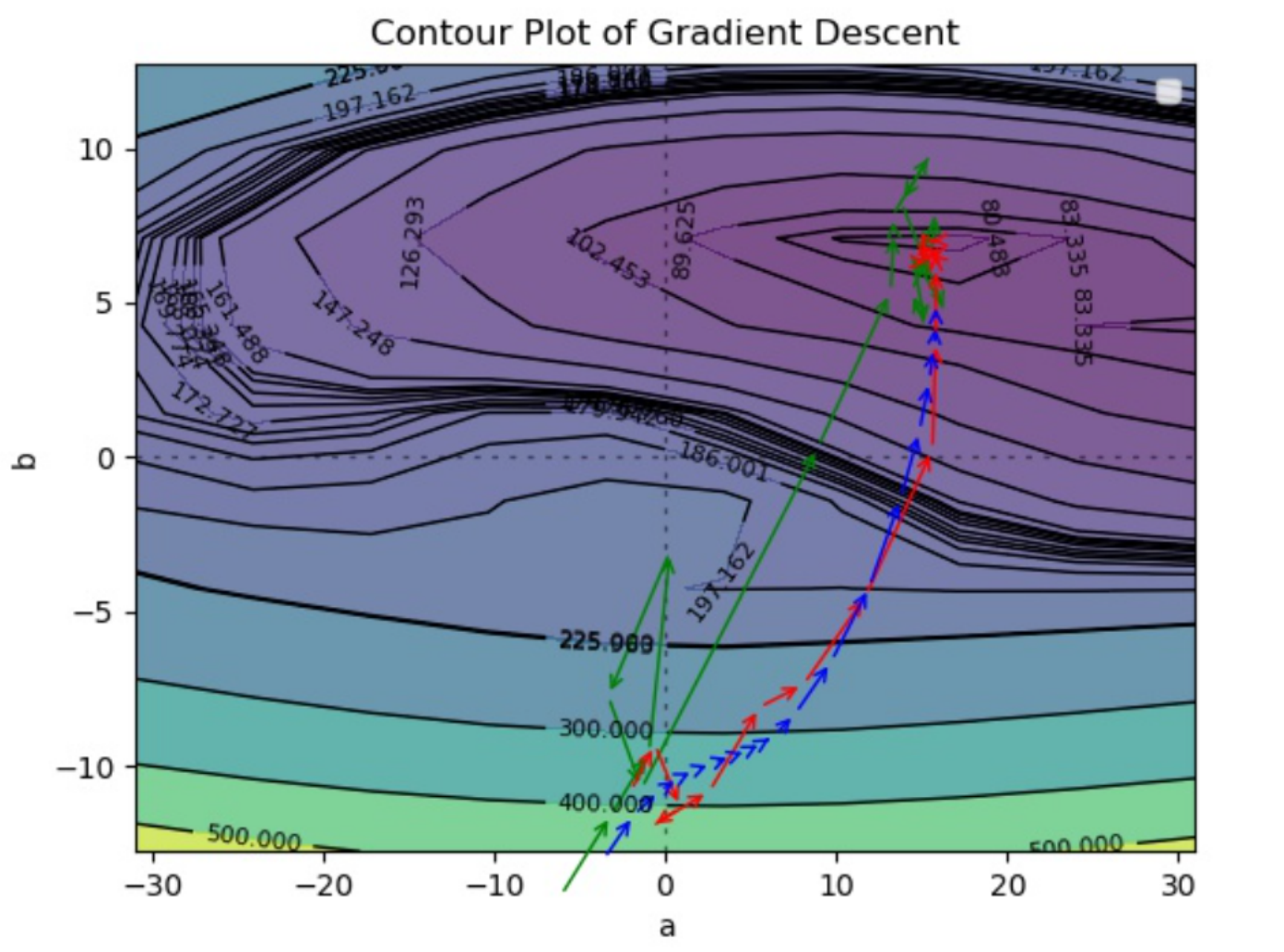}   
	\end{minipage}
}
\subfigure[Trace of optimization of simplified version \\of our method in Section \ref{sec:simplify_lr}.] 
{
    \hspace{-2em}
	\begin{minipage}{7cm}
	\vspace{-0.3em}
	\centering      
	\includegraphics[scale=0.5]{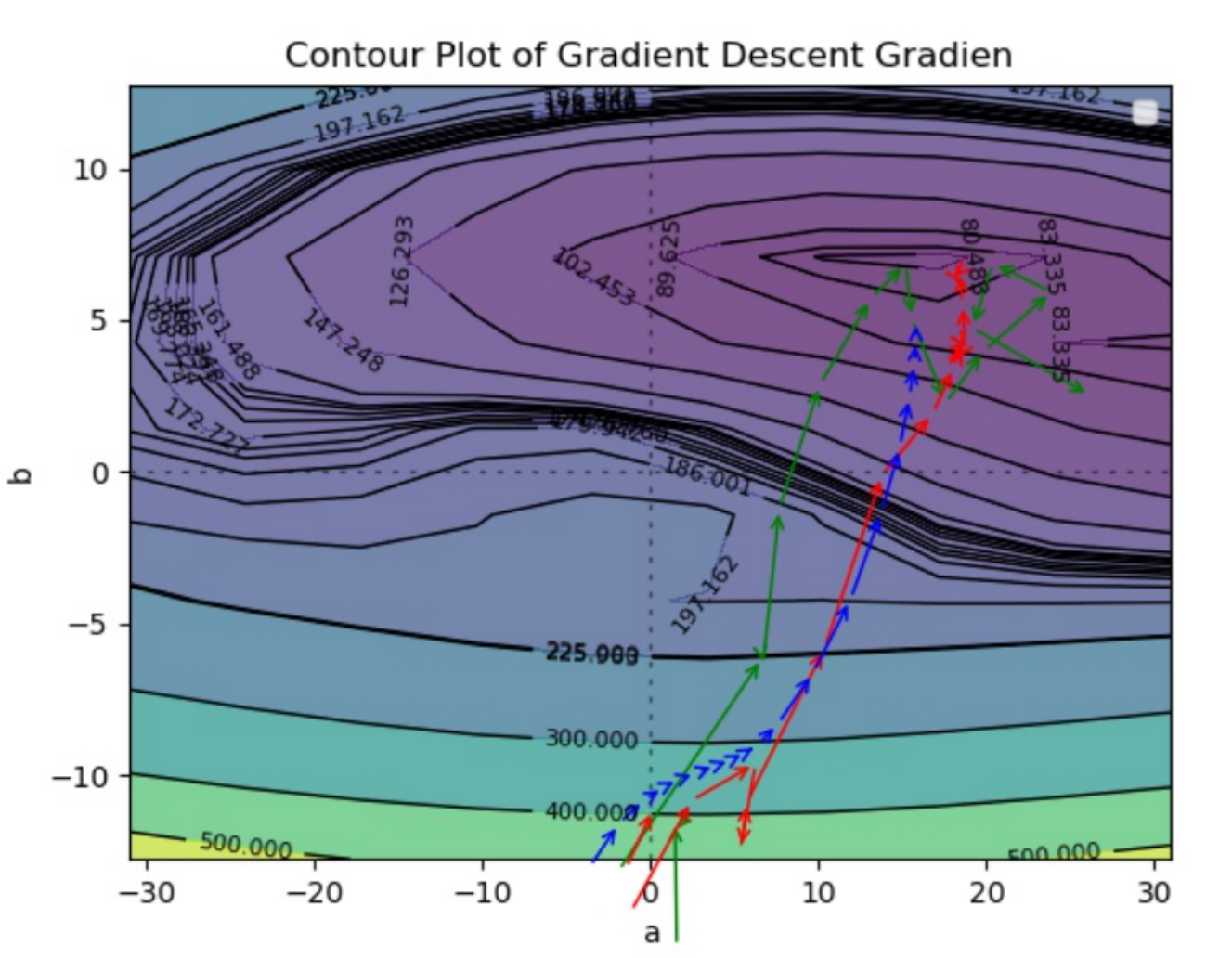}   
	\end{minipage}
}
\caption{Visualization of the optimization process.} 
\label{fig:viz_opt}  
\end{figure}

Trajectories of training schemes using three gradient estimators respectively are presented in the Fig. \ref{fig:viz_opt}. We use in gradient estimator in Eq. (\ref{gradient}) for training ANN in the left figure and use the gradient estimator in Eq. (\ref{sign_grad}) the right figure. The blue line represents the trajectory of training by using BP. 
The green and red lines are trajectories of training by using our methods (\ref{gradient}) and (\ref{sign_grad}) using only $1$ and $10$ copies, respectively. 
All methods get close to the minimum point eventually. The green trajectories in both two figures oscillate wildly, whereas blue trajectories in two figures are close to the red lines.

Now we assume the objective function has a unique equilibrium point $\tilde{\omega}^*\in\Omega$ and discuss the convergence of recursion (\ref{new_sgd}). We define $\mathcal{F}_k=\{\omega_0,\cdots,\omega_k\}$ as the $\sigma$-algebra generated by our algorithm for $k=0,1,\cdots$. Here we introduce some assumptions before the analysis.
\begin{assumption}\label{a1}
The parameter set $\Omega\subset\mathbb{R}^d$ is closed, convex and compact.
\end{assumption}
\begin{assumption}\label{a2}
The cost function $\mathcal{L}^d(\omega_k)$ is continuously differentiable in $\omega_k$, and convex in $\omega_{k,d}$ for all given $\omega_{k,-d}=[\omega_{k,1},\cdots,\omega_{k,d-1},\omega_{k,d+1},\cdots\omega_{k,D}]\in\Omega_{-d}\subset\mathbb{R}^{d-1}$.
\end{assumption}
\begin{assumption}\label{a3}
The step-size sequence $\{\gamma_k\}$ satisfies $\gamma_k>0$, $\sum_{k=0}^{\infty}\gamma_k=\infty$, $\sum_{k=0}^{\infty}\gamma_k^2<\infty$.
\end{assumption}
\begin{assumption}\label{a4}
The loss value is uniformly bounded, i.e., for all $\omega\in\Omega$, $\vert\mathcal{L}_n(\omega)\vert\leq M<\infty$ w.p.1.
\end{assumption}

We expect recursion (\ref{new_sgd}) to track an ODE:
\begin{align}\label{proj_ode}
    \dot{\omega}(t) = \tilde{\Pi}_{\Omega}(\mathcal{J}(\omega(t))),
\end{align}
with $\tilde{\Pi}_{\Omega}(\cdot)$ being a projection function satisfying $\tilde{\Pi}_{\Omega}(\mathcal{J}(\omega(t)))=\mathcal{J}(\omega(t))+p(t)$, where $p(t)\in-C(\omega(t))$ is the vector with the smallest norm needed to keep $\omega(t)$ in $\Omega$, and $C(\omega)$ is the normal cone to $\Omega$ at $\omega$. We first establish the unique global asymptotically stable equilibrium for ODE (\ref{proj_ode}).

\begin{lemma}\label{lemma}
    If Assumptions \ref{a1} and \ref{a2} hold, then $\tilde{\omega}^*$ is the unique global asymptotically stable equilibrium of ODE (\ref{proj_ode}).
\end{lemma}
\begin{proof}
With Assumptions \ref{a1} and \ref{a2}, if $\tilde{\omega}^*\in\Omega^{\circ}$, then $\mathcal{J}(\tilde{\omega}^*)=0$ and $C(\tilde{\omega}^*)=\{0\}$; if $\tilde{\omega}^*\in\partial\Omega$, then $\mathcal{J}(\tilde{\omega}^*)$ must lie in $C(\tilde{\omega}^*)$, so $p(t)=-\mathcal{J}(\tilde{\omega}^*)$.
By the convexity of $\mathcal{L}^d(\omega_k)$ and the assumption that $\tilde{\omega}^*\in\Omega$ is the unique equilibrium point, $\tilde{\omega}^*\in\Omega$ is the unique equilibrium point of ODE (\ref{proj_ode}).
Take $V(x)=\Vert x-\tilde{\omega}^*\Vert^2$ as the Lyapunov function, and the derivative is $\dot{V}(x)=2(x-\tilde{\omega}^*)^{\top}(\mathcal{J}(x)+p(t))$.
By Assumption \ref{a2} and \cite{facchinei2003finite}, $(x-\tilde{\omega}^*)^{\top}\mathcal{J}(x)\leq0$ for any $x\neq\tilde{\omega}^*$.
Since $p(t)\in C(x)$, we have $(x-\tilde{\omega}^*)^{\top}p(t)\leq0$. Therefore, $\tilde{\omega}^*$ is global asymptotically stable by the Lyapunov Stability Theory \cite{liapounoff2016probleme}.
\end{proof}

To prove that recursion (\ref{new_sgd}) tracks ODE (\ref{proj_ode}), we apply a convergence theorem in \cite{borkar1997stochastic} as below:
\begin{theorem}\label{theorem1}
Consider the recursion
\begin{align*}
    \omega_{k+1}=\Pi_{\Omega}(\omega_k+\gamma_k(\mathcal{J}(\omega_k)+\delta_k)),
\end{align*}
where $\Pi_{\Omega}$ is a projection function, $\mathcal{J}$ is Lipschitz continuous, $\{\gamma_k\}$ satisfies Assumption \ref{a1}, $\{\delta_k\}$ is random variable sequence satisfying $\sum_k\gamma_k\delta_k<\infty$, a.s. If ODE (\ref{proj_ode}) has a unique global asymptotically stable equilibrium $\tilde{\omega}^*$, then the recursion converges to $\tilde{\omega}^*$.
\end{theorem}

Next we show that the conditions in Theorem \ref{theorem1} can be verified in Theorem \ref{theorem2}.
\begin{theorem}\label{theorem2}
If Assumptions \ref{a1}, \ref{a2}, \ref{a3} and \ref{a4} hold, then the sequence $\{\omega_k\}$ generated by recursion (\ref{new_sgd}) converges to the unique optimal solution w.p.1.
\end{theorem}
\begin{proof} Recursion (\ref{new_sgd}) can be rewritten as
\begin{align*}
    \omega_{k+1}=\omega_k-\gamma_k\mathcal{J}(\omega_k)+\gamma_k \delta_k,
\end{align*}
where $\delta_k=\mathcal{J}(\omega_k)-\tilde{g}_k$. Let $M_k=\sum_{i=0}^k \gamma_i\delta_i$. Since $\tilde{g}_k$ is the unbiased estimator of $\mathcal{J}(\omega_k)$, $\{M_k\}$ is a martingale sequence. We can verify that it is $L^2$-bounded. With Assumptions \ref{a3} and \ref{a4}, we have
\begin{align*}
    \sum_{i=0}^k\gamma_i^2\delta_i^2\leq M^2\sum_{i=0}^k\gamma_i^2<\infty.
\end{align*}
By noticing $\mathbb{E}[\delta_i|\mathcal{F}_i]=0$, we have 
\begin{align*}
    \mathbb{E}[\gamma_i\delta_i\gamma_j\delta_j]=\mathbb{E}[\gamma_i\delta_i\mathbb{E}[\gamma_j\delta_j|\mathcal{F}_j]]=0,
\end{align*}
for all $i<j$. Thus $\sup_{k\geq0}\mathbb{E}[M_k^2]<\infty$. From the martingale convergence theorem \cite{durrett2019probability}, we have $M_k\rightarrow M_{\infty}$ w.p.1, which implies that $\{M_k\}$ is bounded w.p.1. Now all conditions in Theorem 1 are satisfied. Therefore, it is almost sure that recursion (\ref{new_sgd}) converge to the unique global asymptotically stable equilibrium of ODE (\ref{proj_ode}), which is the equilibrium point $\tilde{\omega}^*$ by the conclusion of Lemma \ref{lemma}.
\end{proof}

We present our proposed method in Alg. \ref{alg}.
\begin{algorithm}
	\renewcommand{\algorithmicrequire}{\textbf{Input:}}
	\renewcommand{\algorithmicensure}{\textbf{Output:}}
	\caption{Noise Injection based Training Scheme}
	\label{alg}
	\begin{algorithmic}
	    \REQUIRE Model parameter $\theta$,  input feature $X$ with target value $O$, the activation function $f(\cdot)$, loss function $\mathcal{L}(\cdot, \cdot)$.
		\STATE Initialize model parameter $\theta$ and set $\sigma$ to 1. 
		\REPEAT
		\STATE $Y$ $\gets$ $\theta X$,
		\STATE Sample the random standard normal noise $Z$ with the same size as $Y$,
		\STATE $Y \gets f(Y + \sigma Z)$,
		\STATE $l \gets \mathcal{L}(Y, O)$,
		\STATE $g \gets$ compute the gradient using original method (\ref{org_grad}) or sign-approximation method (\ref{sign_grad}),
		\STATE update $\theta$ and $\sigma$ using $g$,
		\UNTIL loss value $l$ converges.
		\ENSURE Parameter $\theta$.
	\end{algorithmic}
\end{algorithm}

\subsection{Training for Spiking Neural Networks}
The SNNs are the third generation of neural network models~\cite{MAASS19971659}, which are characteristic of event-driven signal processing~\cite{gerstner2002spiking}. Here we consider the SNNs with Leaky Integrate-and-Fire (LIF) neurons.
We denote $L$ as the number of layers in SNNs and $m_{i}$ as the number of neurons in the $l$-th neural layer,  $l \in [1, 2, ..., L]$. For the time sequence input ${X}^{(t,0)} \in \mathbb{R}^{m_{0}}$, in the $l$-th layer, we have the previous potential ${U}^{(t,l+1)}$ and spike output ${X}^{(t,l+1)}$. 

Suppose we have $N$ inputs for the network, denoted as ${X}^{(t, 0)}(n)$, $n \in [1, 2, ..., N]$, $t \in [1, 2,, ..., T]$. For the $n$-th input, the membrane potential of the $i$-th neuron of the $l$-th layer at $t$-th time stamp  can be given by
\begin{equation}\label{snn_fw}
\begin{aligned}
    u^{(t+1,l+1)}_{i}(n) &= ku^{(t,l+1)}_{i}(n)(1-x^{(t,l+1)}_{i}(n)) + \sum_{j=0}^{m^{l}}\theta^{(l)}_{i.j}x^{(t+1,l)}_{j}(n)+{\sigma}^{(l)}_{i}{\varepsilon}^{(t+1,l)}_{i}(n),\\
    x^{(t+1,l+1)}_{i}(n) &= I(u^{(t+1,l+1)}_{i}(n) - V_{th}),
\end{aligned}
\end{equation}    
where $k$ is the delay factor decided by the membrane time constant, $V_{th}$ is firing threshold of the neuron, and $I$ is the Heaviside neuron activation function. For all time stamps, the potentials are integrated into each neuron. When the current potential passes $V_{th}$, the neuron releases a spike signal to the next layer, and at the same time the membrane potential is reset to zero. All time stamps share the common parameters, including the weight parameters and variances of noises. The spike signal of the last layer at the last time stamp, namely $X^{(T,L)}$, is the final output of SNNs. 
Denote $\tilde{\mathcal{L}}_n(\theta,\sigma)=\mathcal{L}(\theta,\sigma;{X}^{(T, L)}(n),{O}(n))$.
By the LR method, we can compute the estimated gradient for each parameter as follows:
\begin{equation}\label{lr_snn_gradient}
\begin{aligned}
    \frac{\partial \mathbb{E}[\tilde{\mathcal{L}}_n(\theta,\sigma)]}{\partial \theta_{i,j}^{(l)}}&=\mathbb{E}\left[\sum_{t=1}^{T}\tilde{\mathcal{L}}_n(\theta,\sigma) {x}^{(t, l)}_{j}(n)\frac{{\varepsilon}^{(t, l)}_{i}(n)}{\sigma_i^{(l)}}    \right],\\
    \frac{\partial \mathbb{E}[\tilde{\mathcal{L}}_n(\theta,\sigma)]}{\partial \sigma_{i}^{(l)}}&=\mathbb{E}\left[\sum_{t=1}^{T}\tilde{\mathcal{L}}_n(\theta,\sigma)\frac{1}{{\sigma}^{(l)}_{i}}(\varepsilon^{(t,l)}_{i}(n)^2   -1)\right].
\end{aligned}
\end{equation}  
The detailed derivation is presented in the supplementary. The simplification for the gradient estimation using LR method is also applicable in the training process of SNNs. Then the spike signal ${X}^{(t+1, l)}_{i}(n)$ takes values of $\pm1$, and then we only need to compute $\text{sign}\left({\varepsilon}^{(t+1, l)}_{i}(n)\right)$.




\section{Experiment}

\subsection{Settings}
We conduct experiments on MNIST dataset and  Fashion-MNIST dataset. For MNIST, we apply our method to SNN. All the codes are implemented in a computational platform with PyTorch 1.6.0 and Nvidia GeForce RTX 3090. For all experiments, we evaluate classification accuracy.

To test robustness, we adopt several most investigated adversarial attacks, including 1) gradient-based attack: FGSM, BIM, PGD, MIM, which are white-box attacks; 2) optimization-based attack: NAttack which is a black-box attack; 3) input transformation-based attack: DIM and TIM, which are white-box transferable attacks. 

The maximum perturbation on one pixel ranging from $0$ to $1$ is restricted to $0.1$. Specifically, for iterative attack methods including BIM, PGD and MIM, we set the step size as $0.01$ and the maximum number of steps as $15$. For NAttack, we set the number of queries for models as $50$ and the number of samples for gradient estimation as $100$. For DIM and TIM, the momentum is set as $0.9$. We randomly select $1000$ images for adversarial test.

\subsection{Evaluation for SNN on MNIST}
\begin{table}[]
\caption{Evaluation of the adversarial robustness for SNN on MNIST dataset.}
\label{tab:snn_mnist}
\centering
\begin{tabular}{|l|l|l|l|l|l|l|l|l|}
\hline
Method & Ori.          & FGSM & BIM  & PGD  & MIM  & NAttack & DIM  & TIM  \\ \hline
STBP   & 91.8          & 42.6 & 40.2 & 31.5 & 36.0 & 82.0    & 46.8 & 78.2 \\ 
STBP-N & 76.3          & 57.8 & 56.8 & 52.3 & 55.9 & 68.6    & 61.8 & 64.6 \\ 
NIT    & \textbf{92.3} & 81.4 & 84.0 & 82.7 & 83.8 & 85.2    & 83.6 & 84.8 \\ 
NIT-S & 88.4 & \textbf{82.0} & \textbf{84.8} & \textbf{83.6} & \textbf{84.1} & \textbf{85.4} & \textbf{85.0} & \textbf{85.2} \\
\hline
\end{tabular}
\end{table}
We construct a fully connected SNN to classify images in the MNIST dataset. The SNN contains one hidden layer with 50 neurons. The cross-entropy is adopted as the loss function for classification.  We randomly split the entire dataset into training, validation, and testing datasets in a ratio of 7:2:1. 

We report results for four SNN structures trained by three methods: a) STBP; b) STBP-N: use STBP to train SNN with a standard normally distributed noise in each neuron; c) NIT: use our proposed noise injection-based method in Eq. (\ref{gradient}) to train SNN. d) NIT-S: use our proposed approximation method in Eq. (\ref{approx_gradient}) to train SNN.

\begin{table}[]
\caption{Evaluation of the adversarial robustness for SNN on Fashion-MNIST dataset.}
\label{tab:snn_fashion_mnist}
\centering
\begin{tabular}{|l|l|l|l|l|l|l|l|l|}
\hline
Method & Ori.          & FGSM & BIM  & PGD  & MIM  & NAttack & DIM  & TIM  \\ \hline
STBP   & 72.3          & 38.0 & 28.2 & 23.1 & 35.2 & 73.0    & 36.8 & 59.6 \\ 
STBP-N & 34.5          & ~5.0 & ~4.0 & ~4.0 & ~4.0 & 15.0    & ~7.0 & ~5.0 \\ 
NIT    & \textbf{74.6} & \textbf{73.0} & \textbf{72.4} & \textbf{71.8} & \textbf{71.6} & \textbf{73.0}    & \textbf{73.4} & \textbf{73.2} \\ 
NIT-S & 58.7 & 54.4 & 53.6 & 53.2 & 54.1 & 54.0 & 54.4 & 54.4 \\ 
NIT-S+ & 70.6 & 72.1 & 70.8 & 70.8 & 69.5 & 71.9 & 68.3 & 70.5 \\ \hline
\end{tabular}
\end{table}

\textbf{Robustness under white box attack:}
The results are shown in Tab. \ref{tab:snn_mnist}. Compared to SNN trained by STBP, STBP-N can significantly improve robustness against all of the FGSM, BIM, PGD, MIM, NAttack, DIM and TIM attacks, but the accuracy of the original classification task drops significantly. The NIT leads to the best performance in original samples and samples affected by adversarial attacks. Specifically, SNN trained by NIT achieves the optimal classification accuracy of $92.3\%$ in the original samples; moreover, it also achieves a 91.1\% (81.4\% vs 42.6\%) increase in accuracy under the FGSM attack, a 108.9\% (84.0\% vs 40.2\%) increase in accuracy under the BIM attack, a 162.5\% (82.7\% vs 31.5\%) increase in accuracy under the PGD attack, a 132.8\% (83.8\% vs 36.0\%) increase in accuracy under the MIM attack, a 3.9\% (85.2\% vs 82.0\%) increase in accuracy under the NAttack attack, a 78.6\% (83.6\% vs 46.8\%) increase in accuracy under the DIM attack, and 8.4\% (84.8\% vs 78.2\%) increase in accuracy under the TIM attack. 

It is interesting to notice that though NIT-S losses classification accuracy in original samples slightly, it leads to the best results in samples under adversarial attacks and consumes less  memory with much higher computational efficiency. The detailed performance analysis, including the memory consumption, GPU utilization, training and inference time consumption, is supplied in the appendix. 

\subsection{Evaluation for SNN on Fashion-MNIST}


\textbf{Robustness under white box attack:}
The results are shown in Tab.~\ref{tab:snn_fashion_mnist}. The NIT leads to the best performance. Specifically, SNN trained by NIT achieves the optimal classification accuracy of $74.6\%$. It also achieves the highest accuracy under all the adversarial attacks.
The accuracy of NIT and NIT-S is much better than STBP and STBP-S, under adversarial attacks . Notice that compared to NIT, the accuracy of NIT-S drops significantly, which could be explained by the analysis in Section 3.2. As suggested in Section 3.2, we combine NIT and NIT-S, denoted as NIT-S+, to improve performance and computational efficiency simultaneously. 


\subsection{More Discussion on NIT-S}

\begin{figure}[htbp]
\subfigure[The train loss of different methods.] 
{   
    \hspace{-2em}
	\begin{minipage}{7cm}
	\centering          
	\includegraphics[width=\textwidth]{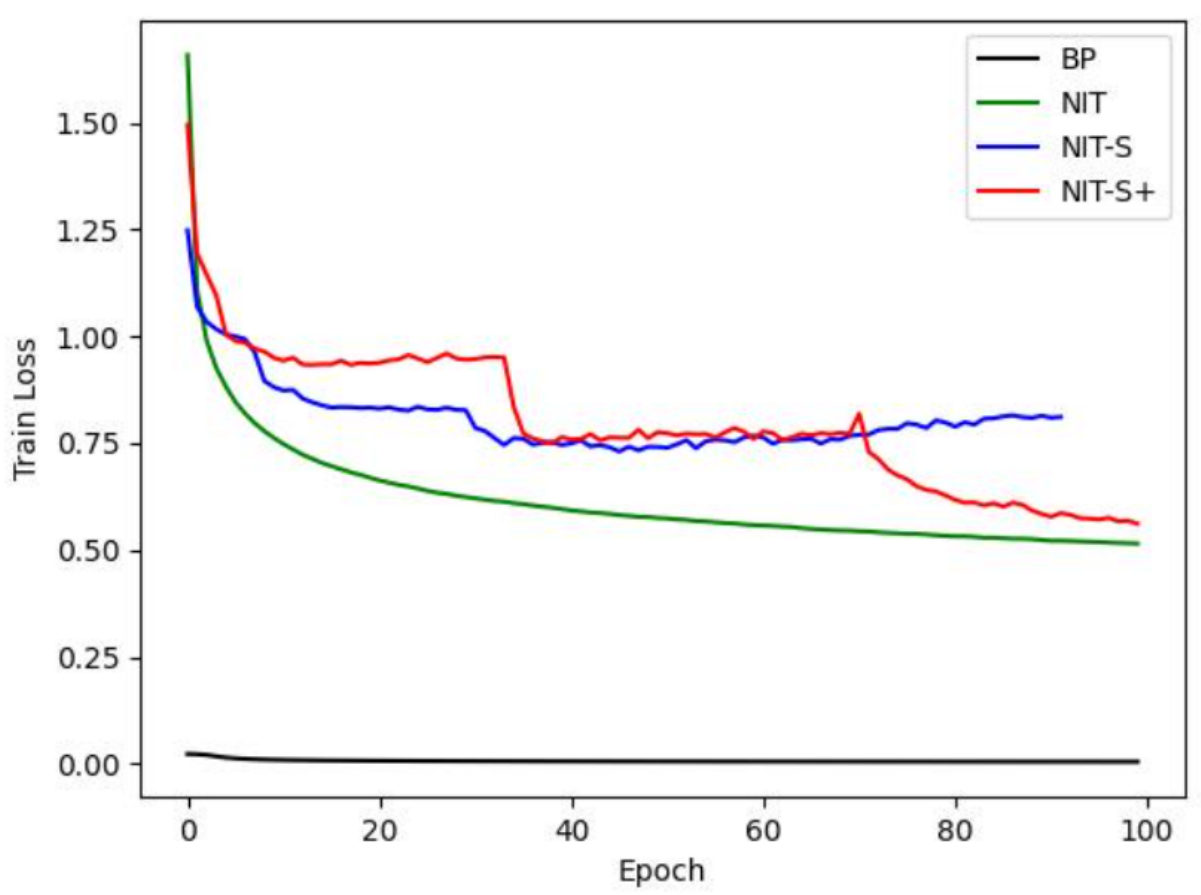}   
	\end{minipage}
}
\subfigure[The testing accuracy of different methods.] 
{
    \hspace{-1em}
	\begin{minipage}{7cm}
	\centering      
	\includegraphics[width=\textwidth]{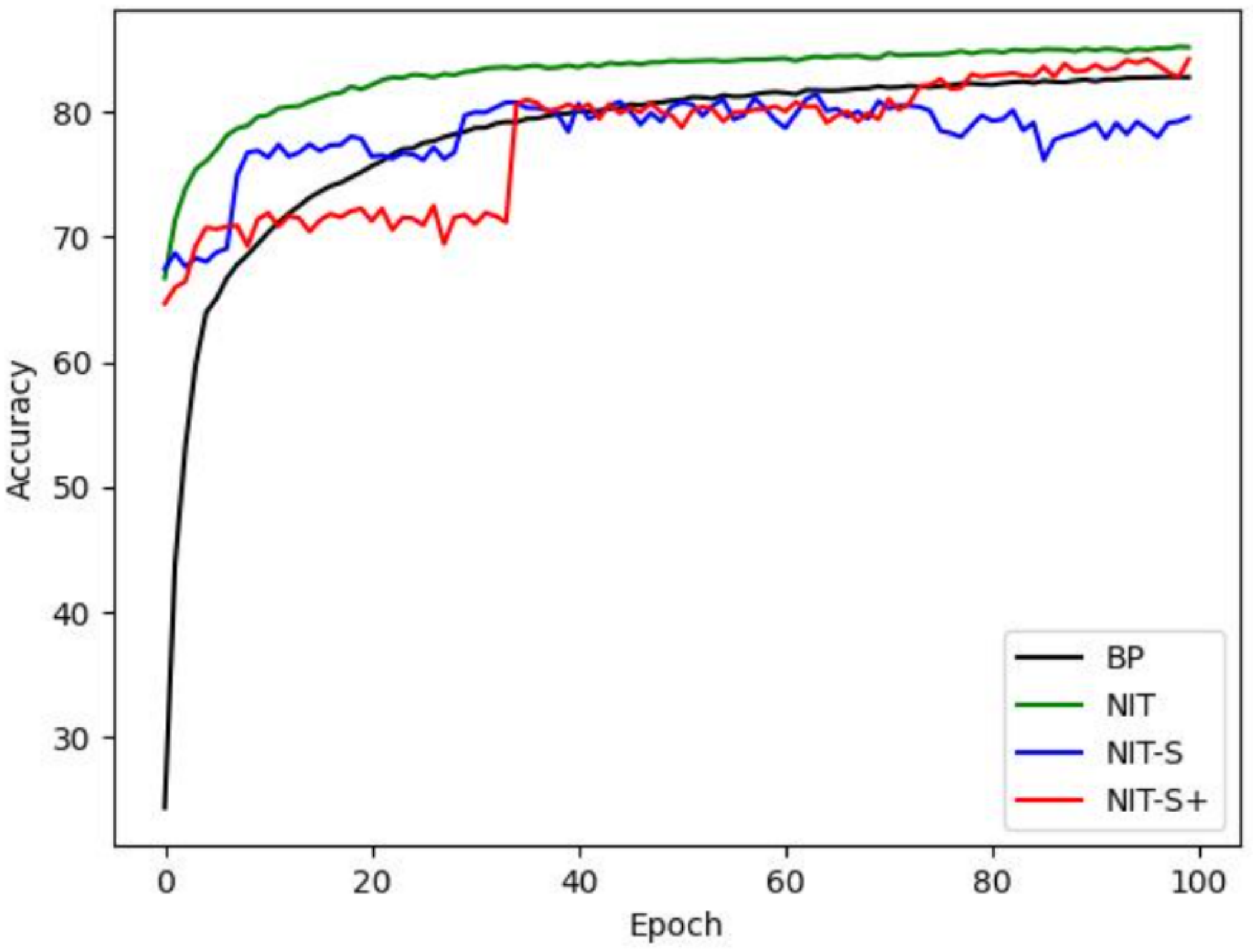}   
	\end{minipage}
}
\caption{The train loss and the testing accuracy.} 
\label{fig:discussion}  
\end{figure}

We use the conventional ANN trained by BP method as the baseline.
Then, we train ANN by NIT-S+, which uses NIT-S for the first $70$ epochs and NIT for the last $30$ epochs. 

As shown in Fig.~\ref{fig:discussion}, the train loss of both NIT-S and NIT-S+ converge to the same level eventually. We note that the accuracy of NIT-S drops after $70$ epochs, whereas NIT-S+ further improves the performance of the model. 
The performance of SNN trained by NIT-S+ on Fashion-MNIST is reported in Tab.~\ref{tab:snn_fashion_mnist}.

\section{Conclusion}

In this work, we propose a novel noise injection-based training method for better robustness. Our method is applied to train SNN in MNIST and Fashion-MNIST datasets. The proposed method significantly improves the performance under various types of adversarial attacks, including gradient-based attack, optimization-based attack, and input transformation-based attack, as well as the accuracy in the original dataset. We also propose a simplified version which applies a sign function on the gradient estimates, which reduces the memory and computation cost. Moreover, the simplified method can be combined together with the originally proposed method to achieves a well-balanced performance on correctness, robustness, and efficiency.

{\small
\bibliographystyle{plainnat}
\bibliography{ref}
}



\end{document}